\documentclass[letterpaper, 10 pt, conference]{ieeeconf}  

\IEEEoverridecommandlockouts                              

\usepackage{subfiles}

\newcommand{\notinsubfile}[1]{}
\usepackage[linesnumbered,lined,ruled]{algorithm2e}
\usepackage[export]{adjustbox}

\usepackage[T1]{fontenc}
\usepackage{amsmath} 
\usepackage{amssymb}  
\usepackage{todonotes}

\usepackage{balance}

\usepackage{amsthm}
\usepackage{graphicx}
\usepackage{subcaption}

\newcommand{\set}[2]{\{#1\: : \: #2\}}
\newcommand{\norm}[1]{\left\lVert #1\right\rVert}

\newcommand{\R}{\mathbb{R}}

\newcommand{\ra}{\rightarrow}

\DeclareMathOperator{\argmin}{arg\,min}

\newtheorem{theorem}{Theorem}

\newtheorem{proposition}{Proposition}
\theoremstyle{definition}
\newtheorem{example}{Example}

\theoremstyle{definition}
\newtheorem{definition}{Definition}
\theoremstyle{remark}
\newtheorem{remark}{Remark}

\title{\LARGE \bf
Model-Free Barrier Functions via Implicit Evading Maneuvers
}

\author{Eric Squires, Rohit Konda, Samuel Coogan, and Magnus Egerstedt$^\dagger$%
\thanks{$^\dagger$Eric Squires is with the Georgia Tech Research Institute.
    Rohit Konda is with the University of California Santa Barbara.
    Samuel Coogan is with the School of Electrical and
    Computer Engineering,
    Georgia Institute of Technology.
    Magnus Egerstedt is with the Samueli School of Engineering
        University of California, Irvine.
    This work was enabled by NSF Grant No. 1836932
    and the University System of Georgia
    Tuition Assistance Program.
}
}

\begin{document}

\maketitle
\thispagestyle{empty}
\pagestyle{empty}

\begin{abstract}
This paper demonstrates that the safety override
arising from the use of a barrier function can in some cases be needlessly restrictive.
In particular, we examine the case of fixed-wing collision avoidance and show
that when using a barrier function, there are cases where two fixed-wing
aircraft can come closer to colliding than if there were no barrier function at
all. In addition, we construct cases where the barrier function
labels the system as unsafe even when the
vehicles start arbitrarily far apart.
In other words, the barrier function ensures safety but with unnecessary
costs to performance. We therefore
introduce model-free barrier functions which take a data driven approach to
creating a barrier function. We demonstrate the
effectiveness of model-free barrier functions in a collision avoidance
simulation of two fixed-wing aircraft.
\end{abstract}

\section{INTRODUCTION}

\label{sec_intro}

Barrier functions \cite{ames2017control}, a function of the state whose
derivative is bounded, can be used to maximize performance 
while ensuring safety. However, if the safety constraint
from the barrier function is
overly restrictive then performance can be diminished. For
example, in adaptive cruise control, safety designers can choose
a minimum inter-vehicle distance that the vehicle must satisfy. Setting
this distance too high will result in
excessive inter-vehicle distances where speed setpoints are difficult to
achieve. In other words, the performance goal (speed) is negatively
impacted by an overly conservative constraint (inter-vehicle distances).

In this paper we show a general solution to this problem and apply
it to 
fixed-wing unmanned aerial vehicle (FW-UAV) collision avoidance.
We first consider
the case where the barrier function ensures each vehicle
can maintain a straight trajectory without collisions.
In this case
even when the vehicles
are arbitrarily far apart
the barrier function can indicate the vehicles
are unsafe, resulting in performance degradation.
For instance, a vehicle located far away could orient itself in a way that makes the barrier function
imply an override is needed.
This can make the system unpredictable as
non-local factors (e.g. vehicles far away) can have an impact on
control choices. This could even be exploited by malevolent actors
who choose to orient their own aircraft in a way
that forces the aircraft to adjust in suboptimal ways.

Another case is a barrier function
that ensures vehicles
can employ a turning maneuver.
We construct a scenario where
using
a nominal controller designed for performance but not safety would result in
vehicle distances far greater than the threshold but a barrier function
results in a significant alteration that causes them to barely exceed the safety distance.
This reduces performance, increases safety risks from unmodelled noise,
and reduces trust as
observers see the safety override causing the vehicles to fly needlessly close.

Prior work has relaxed the override
while ensuring safety by
constructing a barrier function that 
accounts for the nominal controller.
In \cite{wang2018permissive} the authors 
maximize the set of safe states that are compatible
with a region of attraction to maximize performance.
Similarly, a nominal controller and barrier function are learned
simultaneously in \cite{qin2021learning}.
Imitation learning was used in \cite{robey2020learning}
where a barrier function is constructed 
from expert trajectories where the expert
can consider performance and safety
factors.
Barrier functions
have also been used to guide exploration in \cite{cheng2019end} via off
policy reinforcement learning (RL). Similarly, \cite{ma2021model}
introduces a barrier function to constrain
the policy update in RL.

Rather than training both the nominal controller and
safety override, we maximize the set of available safe controls
that could be applied to any nominal policy.
This separates concerns to simplify controller
design \cite{borrmann2015control}.
In particular, we show that
maximizing the set of safe states is not enough to ensure that an
override is not restrictive. In other words, given a state that is safe
for two different barrier functions, it may be that the set of
controls to keep the system safe is larger for a barrier function with a smaller overall safe set.

We also construct a barrier function without requiring a dynamics model
which differs from prior work on barrier functions with uncertainty
\cite{choi2020reinforcement,robey2021learning}.
This can reduce model mismatch
that can lead to
real-world performance degradation.
Further,
model-free approaches can often
outperform model-based systems \cite{nagabandi2018neural} as they are less restricted
in fitting to data.
Finally, the model-free approach of this paper enables a general solution that 
can be applied across a large class of problems with different dynamics
and safety constraints without having to manually re-derive a barrier function.
For instance, while we demonstrate the algorithm using FW-UAV collision avoidance,
the same algorithm could equally be applied to quadrotors.

Thus, we propose model-free barrier functions (MFBFs),
which are learned from interactions with the environment,
to reduce how much the system is overridden.
This approach differs from, for instance, model-free RL as it allows introspection of safety
characteristics to identify why safety override selections are made, whereas
introspection in model-free RL is difficult.
Contributions are the following.
First,
we motivate MFBFs with examples
from FW-UAV collision avoidance \cite{squires2019composition}
that demonstrate a model-based approach induces unnecessary overrides.
Second, we derive MFBFs.
Third, we demonstrate the approach in simulation.
A video of the behavior is available \cite{squires2021modelfreevideo}.
This paper is organized as follows. Section~\ref{sec_model_free_background}
introduces the background for barrier functions. Section~\ref{sec_model_free_theory}
derives MFBFs. Section~\ref{sec_model_free_simulation}
demonstrates the algorithm in simulation. Contents of this paper have
previously appeared in the thesis \cite{squires2021barrier}.

\section{Background}
\label{sec_model_free_background}

In this paper we motivate model-based and model-free barrier functions
with FW-UAV collision avoidance.
Given
two FW-UAVs indexed by $i$ ($i\in\{1,2\}$), vehicle $i$ state
and control inputs are
$x_{k,i} = \begin{bmatrix} p_{k,i,x} & p_{k,i,y} & \theta_{k,i} & p_{k,i,z} \end{bmatrix}^T$
and
$u_{k,i} = \begin{bmatrix} v_{k,i} & \omega_{k,i} & \zeta_{k,i}\end{bmatrix}^T$,
where
$p_{k,i,x}$, $p_{k,i,y}$, and $p_{k,i,z}$ are the $x$, $y$, and $z$ position
while $v_{k,i}$, $\omega_{k,i}$, and $\zeta_{k,i}$ are the translational,
rotational, and vertical velocities with
$v_{min}\le v_{k,i} \le v_{max}$, $v_{min} > 0$,
$|\omega_{k,i}| \le \omega_{max}$,
and 
$|\zeta_{k,i}| \le \zeta_{max}$.
The discrete time dynamics for vehicle $i$ are
\[
x_{k+1,i} = 
\begin{bmatrix}
p_{k,i,x} + v_{k,i} \cos \theta_{k,i}\Delta t \\
p_{k,i,y} + v_{k,i} \sin \theta_{k,i}\Delta t \\
\theta_{k,i} + \omega_{k,i}\Delta t \\
p_{k,i,z} + \zeta_{k,i}\Delta t
\end{bmatrix}.
\]
The two FW-UAV system has state
$x_k = \begin{bmatrix} x_{k,1}^T & x_{k,2}^T \end{bmatrix}^T$
with dynamics of the form
\begin{equation}
x_{k+1} = f(x_k,u_k)
\label{eq_discrete_system}
\end{equation}
where $x_k\in \R^n$, $u_k\in U \subset \R^m$, and $U$ is the set
of available controls for the system.
In the system above of two FW-UAVs, $n=8$ and $m=6$.
We briefly summarize \cite{agrawal2017discrete}, which develops barrier functions
for discrete time with 
dynamics in (\ref{eq_discrete_system}).
Let $h:\R^n\to \R$ be an output function of the state
and define the safe set $\mathcal{C}$ as a superlevel set of
$h$ so that
\begin{equation}
\mathcal{C} = \set{x_k \in \R^n}{h(x_k) \ge 0}.
\label{eq_discrete_safe_set}
\end{equation}
Let $\Delta h(x_{k},u_k) = h(x_{k+1}) - h(x_k)$.
The following definition
is an adaptation from Definition 4 of \cite{agrawal2017discrete}
using terminology similar to \cite{ames2017control}.

\begin{definition}
A map $h:\R^n \to \R$ is a \emph{Discrete-Time Exponential
Control Barrier Function (DT-ECBF)} on a set $\mathcal{D}$ where $\mathcal{C}\subseteq \mathcal{D}$ if
there is a $u_k\in \R^m$ and $\lambda$ such that 
$\Delta h(x_k,u_k) +\lambda h(x_k) \ge 0$
and $0\le \lambda \le 1$
for all $x_k\in\mathcal{D}$.
\label{def_dtecbf}
\end{definition}

The following is an adaptation from Proposition 4 of
\cite{agrawal2017discrete} using the admissible
control space \cite{ames2017control} defined as
\begin{equation}
K(x_k) = \set{u_k\in U}{\Delta h(x_k, u_k) + \lambda h(x_k) \ge 0}.
\label{eq_discrete_admissible_control_space}
\end{equation}

\begin{proposition}
Given a set $\mathcal{C}\subset \R^n$ defined in (\ref{eq_discrete_safe_set})
for an output function $h$, let $h$ be a DT-ECBF on $\mathcal{D}$
and $u:\R^n\ra U$ be such that $u(x_k)\in K(x_k)$
for all $x_k\in D$.
If $x_0\in \mathcal{C}$ then
$x_k\in \mathcal{C}$ for all $k > 0$.
\label{prop_discete_ecbf}
\end{proposition}

If the system has a nominal controller $\hat{u}_k$ that does not necessarily
ensure safety, an optimization can
select a control value $u_k^*$ as close as possible
to $\hat{u}_k$ while ensuring
safety:
\begin{IEEEeqnarray}{rCl}
u_k^* &=& \argmin_{u_k\in U} \frac{1}{2}\norm{u_k - \hat{u}_k}^2 \label{eq_orig_nonconvex_optimization}\\
&\text{ s.t. }& u_k\in U \IEEEnonumber\\
&& u_k\in K(x_k).\IEEEnonumber
\end{IEEEeqnarray}
Equation \eqref{eq_orig_nonconvex_optimization}
is nonconvex \cite{agrawal2017discrete} and we resolve this 
by assuming $U$ is a discrete
set.

\section{Generating a Model-Free Barrier Function via Evasive Maneuvers}
\label{sec_model_free_theory}

\subsection{Constructing Barrier Functions For Discrete Time}
In \cite{squires2019composition} the authors demonstrate how to construct a barrier function for
continuous time systems so we first adapt
that method to discrete time.
A similar approach is \cite{gurriet2018online}
although \cite{squires2019composition} does not require a backup set.
Let $\rho:\R^n\to \R$ be a safety function
that must be nonnegative at all times for the system to be safe.
Let $\gamma:\R^n\to U$ be an evasive maneuver.
Note that $\gamma$ is not the safety override
but instead constructs a barrier function.
A candidate DT-ECBF is the worst case
safety value after forward propagating
the state using $\gamma$. Let 
\begin{equation}
h(x_0) = \inf_{k \ge 0} \rho(\hat{x}_k)
\label{eq_discrete_h}
\end{equation}
where
$\hat{x}_0 = x_0$ and $\hat{x}_{k+1} = f(\hat{x}_{k},\gamma(\hat{x}_{k}))$ for $k > 0$.
In forming a MFBF, we treat $f$ as a black box simulation model.

\begin{theorem}
Given a dynamical system (\ref{eq_discrete_system})
and a function $h$ defined in (\ref{eq_discrete_h})
with a safety function $\rho$ and an evasive maneuver $\gamma$,
$h$ is a DT-ECBF on the set $\mathcal{C}$.
\label{th_discrete_h_zcbf}
\end{theorem}
\begin{proof}
Suppose $x_0\in\mathcal{C}$ so that $h(x_0) \ge 0$.
Then
$\Delta h(x_0, \gamma(x_0))
= \inf_{k\ge 1}\rho(\hat{x}_k) - \inf_{k\ge 0}\rho(\hat{x}_k).$
The right hand side is nonnegative because
it is the subtraction of the infimum of the same function
on different intervals where the first interval
is a subset of the second interval.
Then $\Delta h(x_0,\gamma(x_0)) \ge 0$. Recalling
as well
that $x_0\in\mathcal{C}$ means that $h(x_0) \ge 0$,
this implies that $\Delta h(x_0,\gamma(x_0)) + \lambda h(x_0) \ge 0$.
Then $\gamma(x_0) \in K(x_0)$ so $h$ is a DT-ECBF.
\end{proof}
\begin{remark}
This theorem and proof are similar to 
Theorem 2 of \cite{squires2019composition} but for
discrete time.
Although in Definition~\ref{def_dtecbf} $\mathcal{D}$ can be larger than
$\mathcal{C}$, Theorem~\ref{th_discrete_h_zcbf}
is only valid for $\mathcal{C} = \mathcal{D}$.
See \cite{squires2019composition} for conditions
for $\mathcal{C}\subset \mathcal{D}$ in continuous time.
\end{remark}

\subsection{The Effect of The Evasive Maneuver on Safe Sets}
\label{sec_effect_of_gamma}
While Theorem~\ref{th_discrete_h_zcbf} shows
that $h$ in (\ref{eq_discrete_h}) is a DT-ECBF
and can be used to guarantee safety,
different choices of $\gamma$ can result
in drastically different safe sets.
Consider the two examples given in \cite{squires2019composition}
where
$\rho(x_k) = d_{1,2}(x_k) - D_s$,
$d_{1,2}$ is the distance between the vehicles, and $D_s$
is the safety threshold.
An evasive maneuver where two vehicles
turn at the same rate but have possibly different speeds
is given by
$\gamma_{turn} = \begin{bmatrix} \eta v & \omega & 0 & v & \omega & 0 \end{bmatrix}^T$
where $0 < \eta \le 1$.
A second evasive maneuver where two vehicles that stay
straight for all time is given by
$\gamma_{straight} = \begin{bmatrix} v_1 & 0 & \zeta_1 & v_2 & 0 & \zeta_2 \end{bmatrix}^T.$
We denote $h_{turn}$
and $h_{straight}$
as the $h$ in (\ref{eq_discrete_h}) constructed
from $\gamma_{turn}$ and $\gamma_{straight}$,
respectively.
These evasive maneuvers are considered
in \cite{squires2019composition} because they enable
a closed form solution to (\ref{eq_discrete_h})
so that the barrier function can be
calculated in real-time.
We consider some examples
where the safe set implied by $h_{turn}$ and $h_{straight}$
results
in either an unnecessary override
or labeling states as unsafe that have ample room to avoid a collision.
A graphical view of these scenarios is in
Fig~\ref{fig_diff_safe_sets}. The path
traversed by the vehicles for Example~\ref{ex_unnecessary_override}
is in Fig~\ref{fig_start_on_left}.

\begin{example}\emph{States Are Labelled Unsafe Where Collisions Can Be Avoided.}
\label{ex_unsafe_states}
For $h_{turn}$ consider an initial 
condition where the two vehicles are 
at the same altitude with orientations pointing at each other.
Then no matter how far apart the vehicles start,
(\ref{eq_discrete_h}) yields $h = -D_s$,
implying the initial conditions
are unsafe. This is because $\gamma_{straight}$
implies a future collision.
As the vehicles are placed arbitrarily
far apart, there is time to turn
to avoid a collision. Nevertheless, according
to $h_{straight}$,
this configuration is outside of the safe set.
Note that this scenario has been previously discussed
in \cite{squires2021sensing} where it was shown 
that there does not exist a finite range sensor
to ensure safety given $h_{straight}$.
\end{example}

\begin{example}\emph{An Unnecessary Invasive Override.}
\label{ex_unnecessary_override}
While $h_{turn}$ does not have the issue in Example~\ref{ex_unsafe_states},
there are other initial conditions that lead to 
an unnecessary override with $h_{turn}$.
Suppose the vehicles pass on the
left with a lateral separation of more than
the safety distance but less than four turn radii. Then
if the vehicles continue straight the vehicles will eventually
approach an unsafe condition according to $h_{turn}$
and the overriding
safety controller will induce a large path correction
so that each vehicle can pass on the others' right.
\end{example}
\begin{figure}
\begin{subfigure}{0.45\columnwidth}
\centering
\def\svgwidth{\textwidth}
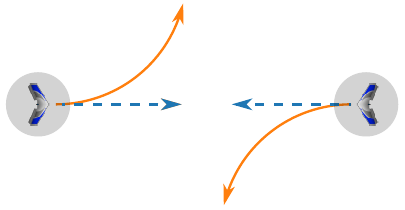
\label{fig_diff_safe_sets_gamma_straight}
\end{subfigure}
\begin{subfigure}{0.45\columnwidth}
\centering
\def\svgwidth{\textwidth}
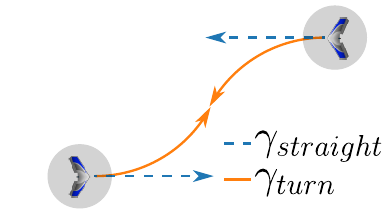
\label{fig_diff_safe_sets_gamma_turn}
\end{subfigure}
\caption{Vehicles are not safe (left) facing each other with $h_{straight}$,
(right) passing on the left with $h_{turn}$.
}
\label{fig_diff_safe_sets}
\end{figure}
\begin{figure}
\centering
\input{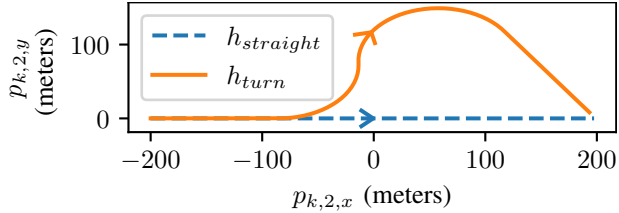}
\caption{%
Given the Fig~\ref{fig_diff_safe_sets} (right) setup,
$h_{turn}$ significantly alters 
the vehicle 2 trajectory but
$h_{straight}$ does not.
}
\label{fig_start_on_left}
\end{figure}
We also 
plot the set of unsafe states for a variety of configurations
(Fig~\ref{fig_safe_sets_stdevs}) to
demonstrate
that even when the vehicles are not pointing at each other,
the vehicles can be spaced far apart and be in an unsafe state
with $h_{straight}$.
Further, Fig~\ref{fig_safe_sets_stdevs} (top left)
shows that the vehicles are unsafe even when
they have flown past each other with $h_{turn}$.
These examples demonstrate cases
where a barrier function results in restrictive
overrides. 
This paper resolves these issues by fitting a barrier function
whose safe set iteratively grows as well as increases
the admissible control space.
The method is not specific to FW-UAV collision avoidance.

\subsection{An Initial Model-Free Barrier Function}
\label{sec_initial_mf}
The issues in Figures~\ref{fig_diff_safe_sets}
and~\ref{fig_safe_sets_stdevs} result because
the $\gamma$ used to calculate $h$ are constant.
While more complicated $\gamma$ may be preferable, it makes
(\ref{eq_discrete_h}) difficult to solve in closed form.
To resolve this,
we propose a data driven approach.
To do so, we start the state at some
$x_0\in \R^n$ and apply some evasive maneuver
$\gamma$.\footnote{%
$x_0$ is sampled from $\R^n$
rather than $\mathcal{D}$ during 
the data-generation phase.
Otherwise the data would have a bias toward safe prediction.
}
If an evasive maneuver has not been specified,
let $\gamma = \hat{u}$.
Given an evasive maneuver,
we create a sequence $\{x_k\}_{k=0}^T$
where $T$ is some horizon over which safety is evaluated.
In the case of FW-UAV collision avoidance, $T$ may represent
battery life of the vehicles after which collisions will not occur.

Note that the sequence $\{x_k\}_{k=0}^T$
is the enumeration of states on the right
hand side of (\ref{eq_discrete_h}).
Thus, given a starting state $x_0$, 
$\rho_{min} = \min_{k\ge 0} \rho(x_k)$ is a sample $h(x_0)$.
Suppose this process is repeated $N$ times
to form a dataset $D = \{(x_0^j,\rho_{min}^j)\}_{j=1}^N$.
Then we 
can fit a function $\hat{h}$ to approximate the mapping (\ref{eq_discrete_h})
with the dataset $D$.
In the perfect case without error
we are left with a function that directly calculates (\ref{eq_discrete_h})
without having to do the integration
because the integration is implicit in the fitting of the data.

However, when fitting $\hat{h}$
there will be errors.
Errors where the learned $\hat{h}$ is less
than the true $h$ leads
to conservative behavior by considering
states to be unsafe that are actually safe.
However, when $\hat{h}$
over predicts, it can imply
the state is safe when it is not.
A conservative approach is to bias
the learned $\hat{h}$ downward to reflect uncertainty.
This can be done by biasing the loss function
\cite{srinivasan2020synthesis}
or alternatively with
a Bayesian approach (e.g., Gaussian Processes 
were used for barrier functions in \cite{wang2018safe}. Bayesian
neural networks \cite{blundell2015weight,gal2016uncertainty} can also output an uncertainty)
by subtracting a desired number of standard deviations (denoted $\sigma$)
from the model output.
We note though that while this method reduces
the chances that the fitted $\hat{h}$ will over predict
the true $h$, because it cannot be guaranteed
this type of error does not occur, the strict
safety guarantee arising from Theorem~\ref{th_discrete_h_zcbf} 
is lost.

\subsection{Iteratively Expanding the Admissible Control Space}

Consider the output of $\hat{h}$ when applied to FW-UAV
collision avoidance with a waypoint following nominal controller.
Position two vehicles arbitrarily far apart with waypoints located at
the starting position of the other vehicle, and orientations
pointing at their respective waypoint.
This configuration will be unsafe for $\hat{h}$
for the same reason as described in Example~\ref{ex_unsafe_states}.
We now show how to improve on this initial estimated $\hat{h}$
with an iterative algorithm.

We examine the case where
a barrier function $h$ is available and
generate a new barrier function $h^1$
with a larger safe set than $h$.
Given $x_0\in \R^n$ and
$\hat{u}_k \in \R^m$,
let $\gamma^{1}:\R^n\ra U$ be the output\footnote{%
Note that because $x_0$ is sampled from $\R^n$ rather
than $\mathcal{D}$ it is not guaranteed that
the optimization program has a solution
when $x_0\notin \mathcal{D}$.
This can be resolved for instance by adding a slack variable.
}
of \eqref{eq_orig_nonconvex_optimization}.
Then $\gamma^{1}$ can be used as an evasive
maneuver since it is a function that maps to the 
action space as required by Theorem~\ref{th_discrete_h_zcbf}.
Thus, we form a new barrier function
$h^{1}$ via (\ref{eq_discrete_h}) with safe set $\mathcal{C}^1$ such that 
\begin{IEEEeqnarray}{C}
h^{1}(x_0) = \min_{k \ge 0} \rho(\hat{x}_k), \label{eq_discrete_overline_h} \\
\hat{x}_{k+1} = f(\hat{x}_{k},\gamma^{1}(\hat{x}_{k})). \label{eq_discrete_overline_xhat_k}
\end{IEEEeqnarray}

\begin{theorem}
Given a dynamical system (\ref{eq_discrete_system})
let $h$ be defined in (\ref{eq_discrete_h})
with safety function $\rho$ and evasive maneuver $\gamma$.
Let
$h^{1}$ be defined in (\ref{eq_discrete_overline_h})
with safety function $\rho$ and evasive maneuver $\gamma^{1}$
defined as the output of (\ref{eq_orig_nonconvex_optimization}).
Then $\mathcal{C}\subseteq \mathcal{C}^1$.
\label{th_discrete_bigger_safe_set}
\end{theorem}
\begin{proof}
Let $x_0\in \mathcal{C}$.
From Proposition~\ref{th_discrete_h_zcbf},
because $\gamma^{1}$ maps to values in
$K(x_k)$ for all $x_k\in \mathcal{D}$,
$\rho(\hat{x}_k) \ge 0$ for $k \ge 0$
where $\hat{x}_k$
is defined in (\ref{eq_discrete_overline_xhat_k}).
Then $h^{1}(x_0) \ge 0$. Then $x_0\in\mathcal{C}^1$.
\end{proof}

Theorem~\ref{th_discrete_bigger_safe_set} says that 
by using $\gamma^{1}$ rather than $\gamma$
as the evasive maneuver, the safe set does
not get smaller.
We now show a case where $\mathcal{C}$ is a strict subset of $\mathcal{C}^1$.

\begin{example}
Consider a discrete double integrator system
\begin{equation}
x_{k+1} =
\begin{bmatrix} 1 & \Delta t \\ 0 & 1 \end{bmatrix} x_k
+ \begin{bmatrix}0 \\ \Delta t \end{bmatrix} u_k,
\end{equation}
where $x_{k,1}, x_{k,2}$ are the position and velocity, respectively.
Let $\rho(x_k) = x_{k,1}$ so the system is point wise safe when
the position is nonnegative, $\gamma(x_k) = 1$,
$\Delta t = 0.1$, and $x_0 = \begin{bmatrix} 0.5 &  -1\end{bmatrix}^T$.
Then $h(x_0) = -0.05$ so $x_0\notin \mathcal{C}$.
In the case where $\hat{u}_k = 2\in U$, the result
of \eqref{eq_orig_nonconvex_optimization}
is $\gamma^1(x_k) = 2$. 
Then using $\gamma^1$ to construct $h^1$ via \eqref{eq_discrete_overline_h},
$h^{1}(x_0) = 0.2$
so $x_0 \in \mathcal{C}^1$.
\label{ex_double_integrator}
\end{example}
The point of Example~\ref{ex_double_integrator}
is that $\gamma^{1}$ can in some cases
do a better job at avoiding unsafe conditions
and as a result the safety set is enlarged.
However, as discussed 
in Section~\ref{sec_initial_mf}, 
to apply 
Theorem~\ref{th_discrete_bigger_safe_set},
one needs to forward propagate
the dynamics (\ref{eq_discrete_overline_xhat_k})
for all future time where the controller at every future timestep is the result
of a nonconvex program (\ref{eq_orig_nonconvex_optimization})
and return the minimum
$\rho(x_k)$ for the resulting sequence $\{\hat{x}_k\}_{k=0}^T$.
For online safety overrides, this is not computationally feasible.
Thus, we pursue the data driven approach discussed
in Section~\ref{sec_initial_mf}. See
Algorithm~\ref{alg_initial_hat_overline_h}.

\begin{algorithm}
\caption{Initial algorithm for learning a MFBF.}
\label{alg_initial_hat_overline_h}
\SetKwInOut{Input}{input}
\SetKwInOut{Output}{output}

\SetKwFunction{FuncExpandSafeSet}{ExpandSafeSet}
\SetKwProg{Fn}{Function}{:}{}

\Input{%
    $h$ (barrier function),
    $N$ (number of samples),
    $\hat{u}$ (nominal controller),
    $T$ (safety horizon)
}
\Output{$\hat{h}^{1}$}

\Fn{\FuncExpandSafeSet{$h$, $\rho$, $N$, $T$}}{%
    $D=\{\}$\;
    \Repeat{repeated $N$ times}{%
        select a random $x_0$\;
        $x\leftarrow x_0$\;
        $\rho_{min}\leftarrow\rho(x)$\;
        \Repeat{repeated $T$ times}{%
            $\gamma^1 \leftarrow $ from equation \eqref{eq_orig_nonconvex_optimization} using $x$, $h$, and $\hat{u}$\nllabel{line_choose_u}\;
            $x\leftarrow f(x,\gamma^{1})$\;
            $\rho_{min}\leftarrow \min(\rho_{min},\rho(x))$\;
        }
        append $\{x_0,\rho_{min}\}$ to $D$\nllabel{line_dataset}\;
    }
    $\hat{h}^{1}\leftarrow$fit to $D$\;
    \KwRet$\hat{h}^1$\;
}

\end{algorithm}

Given Theorem~\ref{th_discrete_bigger_safe_set},
if there are
no errors in fitting $\hat{h}^1$, we expect
that $\mathcal{C}^1$ will be a superset
of $\mathcal{C}$.
However, we can continue this process
to form $\gamma^2$ with the property that
$\gamma^2(x_k) \in K^1(x_k)$ for all $x_k\in \mathcal{C}^1$
where
$K^1(x_k) = \set{u\in U}{\Delta \hat{h}^1(x_k, u_k) + \lambda \hat{h}^1(x_k) \ge 0}.$
See Algorithm~\ref{alg_iterative_h}.
For a barrier function $h^i$ we denote
the admissible control space by $K^i$
and the safe set by $\mathcal{C}^i$.
However, the next example shows that
for $i > j$,
$\mathcal{C}^j \subseteq \mathcal{C}^i$
does not always imply 
$K^j(x_k) \subseteq K^i(x_k)$ for all $x_k\in C^j$.

\begin{example}
Consider the system
in Example~\ref{ex_double_integrator}.
Let $x_0 = \begin{bmatrix} 2 & -1 \end{bmatrix}$,
$\gamma(x_k) = 0.5$, and $\lambda = 0.9$.
Then a numerical calculation 
shows that $K(x_0) = \{u_0\::\: u_0 \ge -3.77\}$.
Let $\hat{u}(x_k) = 1$ if $x_{k,0} = 2$ and $x_{k,1} = -1$,
and $\hat{u}(x_k) = 0.5$ otherwise.
Then $K^1(x_0) = \{u_0\::\: u_0 \ge -3.67\}$.
Thus, although Theorem~\ref{th_discrete_bigger_safe_set}
shows that $\mathcal{C}\subseteq\mathcal{C}^1$,
$K(x_0) \not\subseteq K^1(x_0)$.
\label{ex_kx_not_subset}
\end{example}
Example~\ref{ex_kx_not_subset} shows that even though
the safe set is enlarged when using Algorithm~\ref{alg_iterative_h},
the set of controls available to keep the
system safe may be reduced. This means that there may be a more aggressive
safety override when using $h^1$
rather than $h$.
Thus, we use the maximum
of the barrier functions $h^j$ for $j \le i$ in
Algorithm~\ref{alg_iterative_h}.
Note
that maximums for boolean
composition of barrier functions for 
continuous time systems was analyzed
in \cite{glotfelter2017nonsmooth}. Here we additionally
show that a maximum of barrier functions is a barrier function.

\begin{algorithm}
\caption{Iteratively Expanding the Safe Set}
\label{alg_iterative_h}
\SetKwInOut{Input}{input}
\SetKwInOut{Output}{output}
\SetKwFunction{FuncExpandSafeSet}{ExpandSafeSet}
\SetKwProg{Fn}{Function}{:}{}

\Input{%
    $h$ (barrier function),
    $N$ (number of samples),
    $\hat{u}$ (nominal controller),
    $T$ (safety horizon),
    $L$ (number of expansions)
}
\Output{$\hat{h}^{L}$}
$\hat{h}^0\leftarrow h$\;
\For{$i\leftarrow 1$ \KwTo $L$}{%
    $\hat{h}^i\leftarrow $ExpandSafeSet($\hat{h}^{i-1}$, $\rho$, $N$, $\hat{u}$, $T$)\;
}
\end{algorithm}

\begin{theorem}
Given a dynamical system (\ref{eq_discrete_system})
and DT-ECBFs $h^1$ and $h^2$, the function $h^3$
defined by $h^3(x_k) = \max(h^1(x_k),h^2(x_k))$
is a DT-ECBF on $\mathcal{C}^1\cup \mathcal{C}^2$.
Further,
if $x_0\in C^1\cup C^2$,
$K^1(x_0)\subseteq K^3(x_0)$ or $K^2(x_0)\subseteq K^3(x_0)$.
\label{th_max_is_bf}
\end{theorem}
\begin{proof}
We first prove that $h^3$ is a DT-ECBF on $\mathcal{C}^1\cup\mathcal{C}^2$.
Suppose $x_0\in\mathcal{C}^1\cup\mathcal{C}^2$
and without loss of generality, assume $h^1(x_0) \ge h^2(x_0)$
so $h^3(x_0) = h^1(x_0)$.
Suppose $u_0\in U$ satisfies $\Delta h^1(x_0,u_0) + \lambda h^1(x_0) \ge 0$
and let $x_1 = f(x_0, u_0)$.
Such a $u_0$ exists because $h^1$ is a DT-ECBF.
Then
\begin{IEEEeqnarray*}{l}
\Delta h^3(x_0, u_0) + \lambda h^3(x_0) \\
\qquad= [\max(h^1(x_1),h^2(x_1)) - \max(h^1(x_0),h^2(x_0))] \\
\qquad \qquad + \lambda \max(h^1(x_0), h^2(x_0)) \\
\qquad = \max(h^1(x_1),h^2(x_1)) - h^1(x_0) + \lambda h^1(x_0).\IEEEyesnumber
\label{eq_max_h_barrier_constraint}
\end{IEEEeqnarray*}
\emph{Case 1: }
If $h^1(x_1) \ge h^2(x_1)$ then (\ref{eq_max_h_barrier_constraint})
becomes 
$\Delta h^3(x_0, u_0) + \lambda h^3(x_0) = \Delta h^1(x_0,u_0) + \lambda h^1(x_0) \ge 0.$

\noindent\emph{Case 2: }
If $h^1(x_1) < h^2(x_1)$ then (\ref{eq_max_h_barrier_constraint})
becomes
\begin{IEEEeqnarray*}{rCl}
\Delta h^3(x_0, u_0) + \lambda h^3(x_0) &=&
h^2(x_1) - h^1(x_0) + \lambda h^1(x_0) \\
&\ge& h^1(x_1) - h^1(x_0) + \lambda h^1(x_0) \\
&=& \Delta h^1(x_0, u_0) + \lambda h^1(x_0)\ge 0. 
\end{IEEEeqnarray*}
Then $h^3$ is a DT-ECBF.
This also establishes
$K^1(x_0) \subseteq K^3(x_0)$ on $\mathcal{C}^1\cup \mathcal{C}^2$
if $h^1(x_0) \ge h^2(x_0)$.
By the same logic, for $h^2(x_0) \ge h^1(x_0)$ with $x_0\in \mathcal{C}^1\cup\mathcal{C}^2$,
$K^2(x_0) \subseteq K^3(x_0)$.
\end{proof}
\vspace{-1em}
\begin{remark}
The optimization (\ref{eq_orig_nonconvex_optimization})
is non-convex so finding an online solution may infeasible.
A direct solution to this is
to assume $U$ is a small finite set
so (\ref{eq_orig_nonconvex_optimization})
can be solved with an exhaustive search.
However, when $h^1$ is defined via (\ref{eq_discrete_h})
for some $\gamma$, Theorem~\ref{th_discrete_h_zcbf}
demonstrates that $\gamma$ is always a feasible solution
of (\ref{eq_orig_nonconvex_optimization}) provided $h^1(x_k) \ge 0$
(and similarly for an evasive maneuver used to construct $h^2$ for $x_k\in\mathcal{C}^2$).
Because $K^1(x_k) \subseteq K^3(x_k)$ for all $x_k\in \mathcal{C}^1$,
this means that $\gamma$ is a feasible solution for
(\ref{eq_orig_nonconvex_optimization}) when using $h^3$
and $x_k\in\mathcal{C}^1$.
\end{remark}

The proof of Theorem~\ref{th_max_is_bf} showed that for $h^1(x_0) \ge h^2(x_0)$,
$K^1(x_0)\subseteq K^3(x_0)$ and we now show an example
where the set inclusion is strict.
In other words, by taking the maximum of two barrier functions, we can not only
expand the safe set but also expand the admissible control space.

\begin{example}
Consider again the system in Example~\ref{ex_double_integrator}
with the given $\rho$, $\lambda = 0.9$, and 
$x_0 = \begin{bmatrix} 2 & -1 \end{bmatrix}^T$.
Let $h^1$ and $h^2$ be as defined in (\ref{eq_discrete_h})
where $\gamma^1$ is defined by $\gamma^1(x_k) = 1$
and $\gamma^2$ is defined by $\gamma^2(x_k) = 5$
if $x_{k,0} \le 0.5$ and 0 otherwise.
Let $h^3$ be defined by $h^3(x_0) = \max(h^1(x_0), h^2(x_0))$.
Then a numerical calculation shows $h^1(x_0) = 0.45$,
$h^2(x_0) = 0.25$, $K^1(x_0) = \{u\::\: u \ge -2.56\}$,
$K^2(x_0) = \{u\::\: u \ge -6.78\}$, and $K^3(x_0) = \{u\::\: u \ge -6.48\}$.
In other words, $h^1(x_0) > h^2(x_0)$ and $K^1(x_0) \subset K^3(x_0)$.
\label{ex_strict}
\end{example}

\subsection{Practical Algorithm}

Here we discuss two updates to Algorithms~\ref{alg_initial_hat_overline_h}
and \ref{alg_iterative_h} to enable computationally and memory efficient model-free overrides.
First, while $h^L$ in Algorithm~\ref{alg_iterative_h}
is model-free,
a model is still required 
to use $h^L$ to compute an override.
This is because computing a solution to (\ref{eq_orig_nonconvex_optimization})
requires a calculation of $\Delta h(x_k,u_k)$ which necessitates
a model for the dynamics.
Thus, to make the final result of Algorithm~\ref{alg_iterative_h}
model-free we must also create a learned
function $\Delta \hat{h}$ in Algorithm~\ref{alg_initial_hat_overline_h}. 
To do so, record
the minimum $\rho_{min,1} = \rho(x_k)$ for $x_k = 1,\ldots,T$
in Algorithm~\ref{alg_initial_hat_overline_h}
and train $\Delta \hat{h}$ to predict $\rho_{min,1} - \rho_{min}$ given $x_0$ and $u_0$.
When Algorithm~\ref{alg_initial_hat_overline_h} outputs these two functions,
$\hat{h}$ and $\Delta \hat{h}$, a model-free override can be computed in (\ref{eq_orig_nonconvex_optimization}).

Second, the result of Algorithm~\ref{alg_iterative_h} is a set
of $L$ barrier functions. Theorem~\ref{th_max_is_bf}
says we can take the maximum of these $L$ barrier
functions to iteratively enlarge both the
safe set and admissible control space.
However,
this implies that $L$ barrier functions
must be maintained,
which implies
memory growth and reduces online computation
capability because $L$ models must be queried
at every step.
Thus, to avoid memory growth and improve online
computation,
we can instead adjust the dataset
of Algorithm~\ref{alg_initial_hat_overline_h}
in line~\ref{line_dataset}
as follows:

\par\noindent\rule{0.95\columnwidth}{0.4pt}\\
\begin{small}
\textbf{\ref{line_dataset}}
\end{small}
\qquad\qquad append $\{x_0, \max(h(x_0),\rho_{min})\}$ to $D$;
\vspace{-0.3em}
\par\noindent\rule{0.95\columnwidth}{0.4pt}

%
%

\section{SIMULATION EXPERIMENTS}
\label{sec_model_free_simulation}

We now validate the approach of 
Algorithm~\ref{alg_initial_hat_overline_h}.
We restrict the action space of both vehicles to 
$[-12, 0, 12]$ degrees per second for $\omega$
while holding velocity fixed at $15$ m/s
and altitude rate at $0$.
The initial state for each vehicle is between
$\begin{bmatrix} -200 & -200 & -\pi & 0\end{bmatrix}^T$
and 
$\begin{bmatrix} 200 & 200 & \pi & 0\end{bmatrix}^T$.
Let $\rho(x) = \max(50,d_{1,2}(x) - D_s)$
where $d_{1,2}$ is the distance between the vehicles
and the max
simplifies data normalization. Note that this clipping does not change
$\mathcal{C}$. We let $D_s=25$,
used a learning rate of $1e-4$, 10000 epochs per iteration,
50\% dropout rate, $50$ samples to calculate $\sigma$,
and had $4$ layers of $1024$ nodes with relu activation. We trained the network with
a mean squared error loss.
To form an initial $h$, we ran 50,000 episodes using a waypoint following controller
without a barrier function and fit a mapping of the initial state to closest vehicle distance
for each episode.
Training statistics are in Fig~\ref{fig_training_data}.
During training, the percent of cases where the output minus $3\sigma$
is above the true value in the validation set is between $1$ and $2.5$ percent.

Fig~\ref{fig_safe_sets_stdevs} shows the unsafe set for
the mean value of the MFBF and
when 3$\sigma$ is subtracted.
The latter results in a larger unsafe set.
Fig~\ref{fig_safe_sets_iter} plots how the unsafe set 
is enlarged as the algorithm proceeds.
For iterations 1 to 5, we start each episode so that the barrier function is nonnegative.
The system with a nominal controller alone
had a collision rate of $(8.9, 8.8, 8.9, 8.8, 9.0)$ percent
vs the collision percentages of the system with the MFBF of $(0.0, 0.5, 0.8, 0.4, 0.5)$ percent
so
the number of collisions when using a MFBF
is less than 10\% of the nominal controller.
Additionally note that there are not zero collisions when using a MFBF as there is noise in fitting to the data.
Nevertheless, safety is significantly improved
over using the nominal controller alone.


\begin{figure}
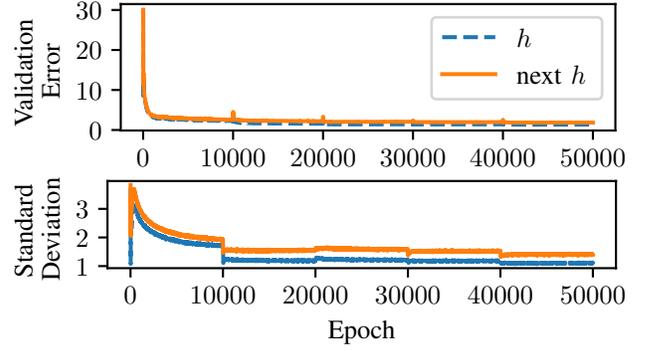

\centering
\begin{subfigure}{\columnwidth}
\centering
\input{imgs/exp/val_errs.pgf}
\end{subfigure}

\begin{subfigure}{\columnwidth}
\centering
\input{imgs/exp/stdevs.pgf}
\end{subfigure}

\caption{%
MFBF validation error (top)
and $\sigma$ (bottom).
}
\label{fig_training_data}
\vspace{-1em}
\end{figure}

\begin{figure}
\centering
\includegraphics[scale=0.8]{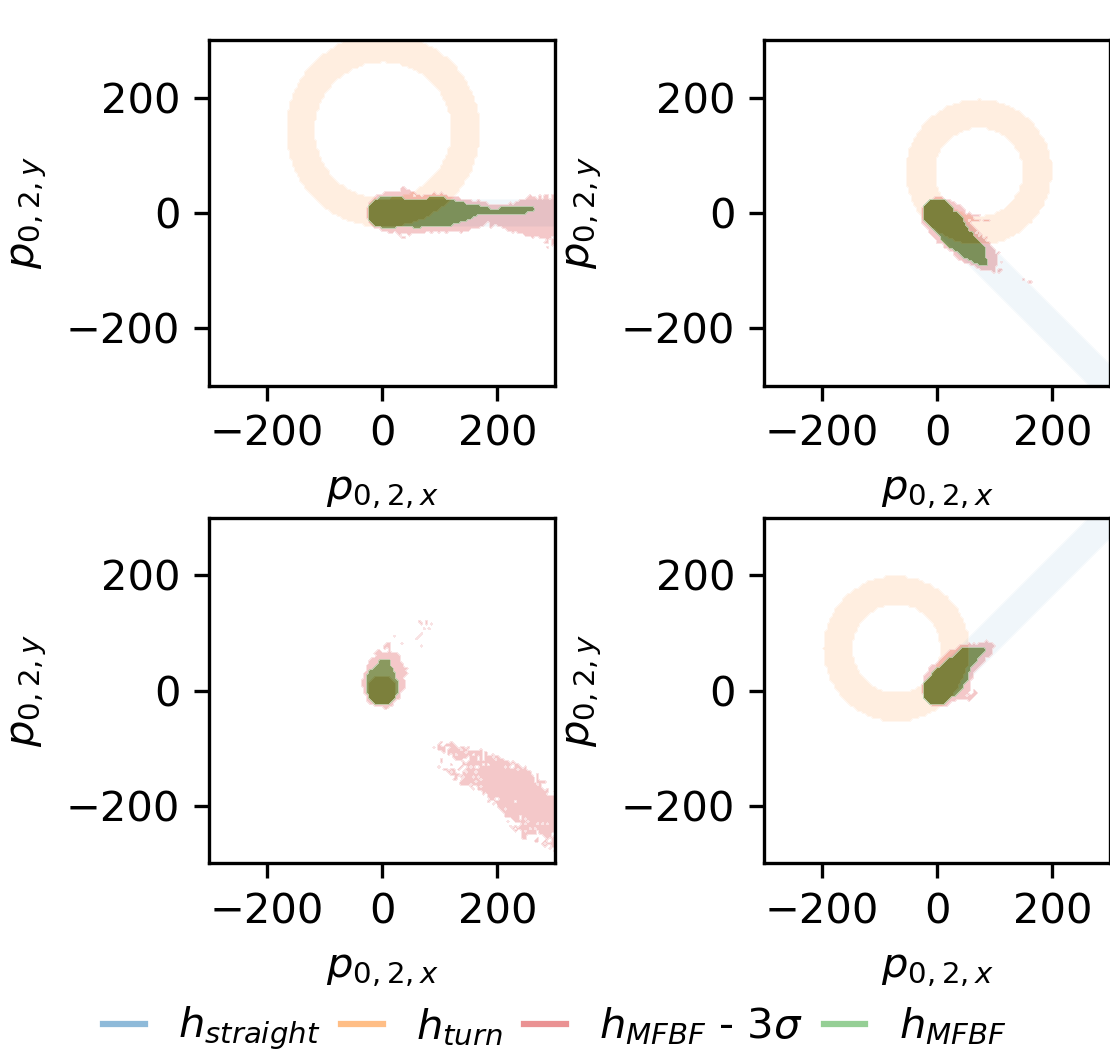}
\caption{
Points where $h(x) < 0$
given $x_1 = [ 0 \quad 0 \quad 0 \quad 0]^T$ (vehicle 1 is at the origin pointing right)
and 
vehicle 2 positions vary. Vehicle 2 orientation is left (top left), up (top right), right (bottom left), down (bottom right).
Training data was sampled from horizontal positions $(-200, -200)$ to $(200, 200)$
so out-of-sample points have higher uncertainty
causing more unsafe states.
}
\label{fig_safe_sets_stdevs}
\vspace{-1.0em}
\end{figure}

\begin{figure}
\centering
\includegraphics[scale=0.8]{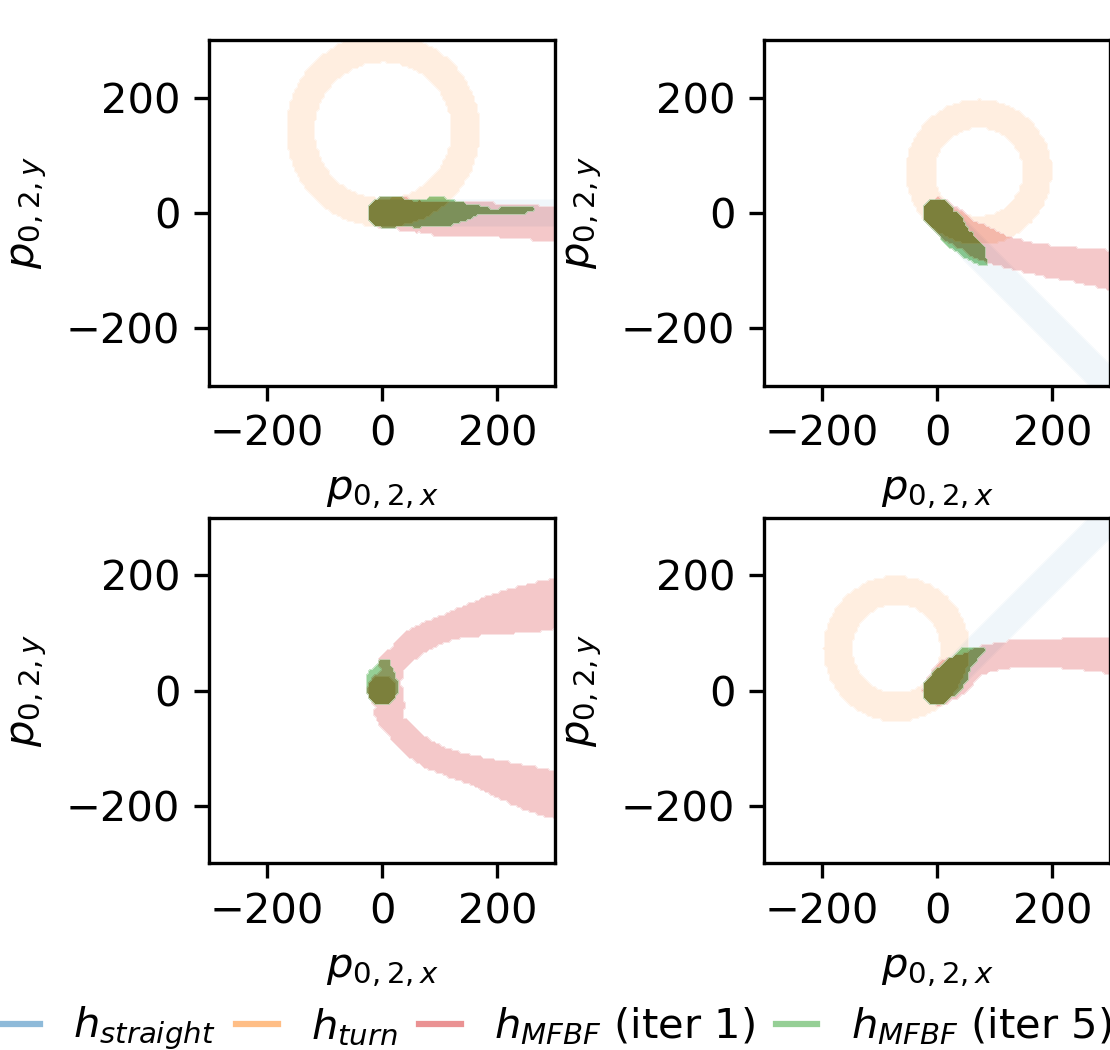}
\caption{%
The same setup as Fig \ref{fig_safe_sets_stdevs} but showing
how the unsafe set changes during training.
As predicted by Theorem~\ref{th_max_is_bf}, the MFBF unsafe set is smaller at iteration 5 than iteration 1.
}
\label{fig_safe_sets_iter}
\vspace{-1.0em}
\end{figure}

\section{CONCLUSION}

\label{sec_conclusion}

In this paper we discussed a few issues with model-based barrier
functions: they may label safe states as unsafe
(Example~\ref{ex_unsafe_states}), cause
unnecessary overrides that cause the state to get closer to the boundary of the
safe set than without an override (Example~\ref{ex_unnecessary_override}),
be difficult to solve for
a barrier function in closed form for complex systems ($h_{turn}$ and $h_{straight}$
exist due to closed form solutions but lead to large unsafe sets, see Fig.~\ref{fig_safe_sets_stdevs}),
and
be numerically infeasible to solve for
a barrier function when there is a long horizon (eq. (\ref{eq_discrete_h})).
Thus,
we introduced MFBFs
which take a data-driven approach
to developing a barrier function.
The tradeoff is that because the barrier
function cannot perfectly fit to the data,
safety guarantees are lost but the benefit
is that the safety set may be significantly
enlarged (Fig.~\ref{fig_safe_sets_stdevs}).
We demonstrated the efficacy of
the approach in a FW-UAV collision
avoidance scenario where, because of the
MFBF, the safety
of the system is significantly improved
over using a nominal controller alone.

\addtolength{\textheight}{-0cm}

\vspace{-0.5em}
\bibliographystyle{IEEEtran}
\bibliography{barrier}

\end{document}